\documentclass{article}

\usepackage{microtype}
\usepackage{graphicx}
\usepackage{subcaption}
\usepackage{booktabs} 
\usepackage{multirow}
\usepackage{tocbibind}
\usepackage{titletoc}
\usepackage{hyperref}

\usepackage[preprint]{icml2026}

\usepackage{amsmath}
\usepackage{amssymb}
\usepackage{mathtools}
\usepackage{amsthm}
\usepackage[table]{xcolor}
\usepackage{array}
\usepackage{booktabs}
\definecolor{darkgreen}{RGB}{0,100,0}
\usepackage{mdframed}
\usepackage[most]{tcolorbox}
\usepackage[capitalize,noabbrev]{cleveref}
\usepackage{xcolor}
\usepackage{colortbl}
\definecolor{lightgray}{gray}{0.92}
\definecolor{lightblue}{RGB}{230,240,255}

\theoremstyle{plain}
\newtheorem{theorem}{Theorem}[section]

\theoremstyle{definition}

\newtheorem{assumption}[theorem]{Assumption}
\theoremstyle{remark}

\usepackage[textsize=tiny]{todonotes}

\icmltitlerunning{Subject-Centric Progressive Visual Token Reduction for VLMs}

\begin{document}

\twocolumn[
  \icmltitle{Focus-then-Context: Subject-Centric Progressive Visual Token Reduction for Vision-Language Models}



  \icmlsetsymbol{equal}{*}

  \begin{icmlauthorlist}
    \icmlauthor{Yulin Zhao}{sch}
    \icmlauthor{Zheng Zhang}{sch,slai}
  \end{icmlauthorlist}

  \icmlaffiliation{sch}{Harbin Institute of Technology, Shenzhen, China}
  \icmlaffiliation{slai}{ShenZhen Loop Area Institute}

  \icmlcorrespondingauthor{Zheng Zhang}{darrenzz219@gmail.com}

  \icmlkeywords{Machine Learning, ICML}

  \vskip 0.3in
]



\printAffiliationsAndNotice{}  

\begin{abstract}
Vision-Language Models (VLMs) face a bottleneck of prohibitive computational costs arising from massive visual token sequences during inference. Existing vision token reduction methods alleviate this burden, but they unintentionally preserve the isolated visual subject strictly aligned with the user's query, which fails to substantially explore salient subjects and their contextual relationships. In this paper, we propose \textbf{SPpruner}, a subject-centric progressive reduction paradigm that emulates the \textit{Focus-then-Context} mechanism of the human visual perception system. Specifically, we first construct a focus identification module to explicitly model the interplay between visual saliency and semantic relevance. Herein, it can excavate the comprehensive visual subject spectrum to ensure a high-fidelity representation of visual input. Subsequently, a context-aware structural scanning module is developed to aggregate contextual cues from neighboring regions. As such, it can effectively restore global relational dependencies to uphold the structural integrity of the preserved subjects. Extensive experiments demonstrate that our paradigm consistently outperforms SOTA methods, achieving up to a 2.53× speedup with only 22.2\% of visual tokens retained in Qwen2.5-VL and a 67\% FLOPs reduction on LLaVA with a negligible 0.6\% accuracy drop.
\vspace{-2ex}
\end{abstract}

\begin{figure*}[t]
  \begin{center}
  \centerline{\includegraphics[width=0.99\textwidth]{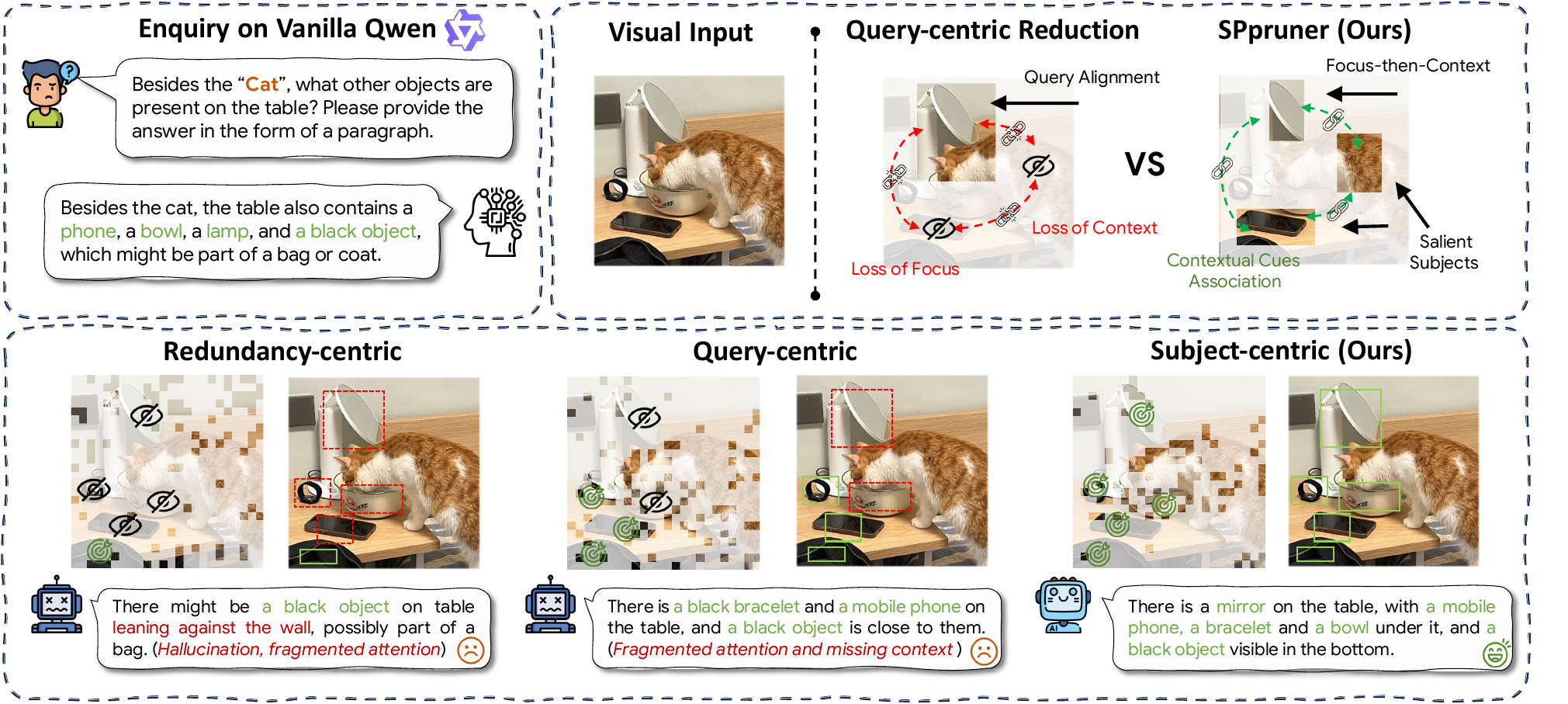}}    
    \caption{\textbf{Comparison between existing paradigms and SPpruner}. Query-centric paradigms tend to retain tokens strictly aligned with explicit text queries and inadvertently discard other salient subjects (e.g., \textit{mirror}) that are critical for answering comprehensive questions. This loss of other salient subjects serves visual understanding. In contrast, SPpruner preserves a broad spectrum of visual subjects and their structural context, facilitating comprehensive visual reasoning. \textcolor{red}{Red} and \textcolor{darkgreen}{green} indicate missed and captured subjects, respectively.}
    \label{fig:motivation}
  \end{center}\vspace{-4ex}
\end{figure*}

\section{Introduction}
Vision-Language Models have demonstrated remarkable success across a broad range of vision tasks~\citep{vqamllms, reasoningmllms}. The core component, the vision encoder, enables the model to map images or videos into a long sequence of visual tokens, enhancing its understanding of visual content. However, this comes at a high computational cost due to the massive length of visual token sequences, which becomes particularly burdensome when processing long-sequence videos~\citep{videounderstanding} or high-resolution images~\citep{highresolutionimage}. This kind of inference paradigm severely restricts their efficiency, making practical deployment difficult in resource-constrained environments.

To mitigate this challenge, Token Reduction (TR) has been proposed to accelerate the inference procedure. Existing TR paradigms generally fall into two categories: (1) \textit{redundancy-centric reduction}, such as ToMe~\citep{tome}, Visionzip~\citep{visionzip}, and FastV~\citep{fastv}, which merge or prune tokens based on statistical redundancy or spatial similarity, and (2) \textit{query-centric reduction}, exemplified by CDpruner\citep{cdpruner} and SparseVLM~\citep{sparsevlm}, which selectively retains tokens solely in response to their relevance to user queries.

Despite the efficiency gains, these methods exhibit a fundamental limitation: they treat TR as a statistical filtering task rather than a subject perception process. As illustrated in Figure~\ref{fig:motivation}, existing methods often fail to adequately preserve tokens that cover all salient visual subjects and their contextual interactions. For instance, query-centric paradigms tend to retain tokens strictly aligned with explicit text queries and inadvertently discard other salient subjects (e.g., \textit{mirror}) that are critical for answering comprehensive questions. This aggressive reduction is detrimental because comprehensive scene understanding often requires contextual information beyond the explicit query terms. Our analysis identifies two primary causes for this shortfall, both stemming from information loss during the token reduction procedure:

(1) \textbf{Loss of Focus:} Existing paradigms relying solely on query-based attention typically capture only the most isolated relevant regions referred to in the given text query. They inadvertently suppress other visually salient regions that are critical for global perception, severely limiting the model's performance in complex multi-subject scenarios.

(2) \textbf{Loss of Context:} By treating token reduction as a local filtering task, current paradigms neglect the structural information of the initial visual input. This results in a set of disjoint visual tokens lacking the necessary environmental context, which makes it difficult to deduce causal or spatial relationships among the retained visual subjects.

In this work, we propose a novel subject-centric progressive token reduction paradigm, named \textbf{SPpruner}. It is designed to address the limitations of prior approaches by incorporating the \textit{Focus-then-Context} principle of the Human Visual Perception System (HVPS) ~\cite{hvs1,hvs2}, where primary visual subjects are identified before analyzing their surrounding contextual relationships. Unlike traditional approaches that rely solely on attention or relevance, SPpruner takes into account both intrinsic visual saliency and semantic relevance, as illustrated in Figure~\ref{fig:framework}(a), allowing it to precisely mine a broader spectrum of salient subjects and ensuring the retention of diverse visual components beyond mere query alignment. Subsequently, as depicted in Figure~\ref{fig:framework}(b), we employ a context-aware structural scanning module to aggregate contextual cues from neighboring regions associated with preserved subjects, empowering the model to effectively uphold structural fidelity and sustain holistic image understanding, even under aggressively high token compression rates. 

In summary, our main contributions are as follows: \vspace{-2ex}
\begin{itemize}
    \item We propose SPpruner, a novel subject-centric progressive token reduction paradigm that simulates HVPS to precisely capture a broader spectrum of salient visual subjects along with their contextual associations, reconciling the conflict between efficient inference and comprehensive visual understanding. \vspace{-1ex}
    \item We construct a focus identification module to enhance the model's capability in capturing diverse visual subjects, and develop a context-aware structural scanning module to aggregate contextual cues and restore global relationships associated with these subjects. \vspace{-1ex}
    \item SPpruner is a plug-and-play solution requiring no additional training or fine-tuning. Extensive experiments on 22 benchmarks demonstrate that SPpruner incurs merely a 2.3\% performance loss at 88.9\% reduction on LLaVA-Next and achieves a 62\% reduction in GFLOPS with only a 5\% performance drop, validating its potential for efficient real-world deployment. \vspace{-1ex}
\end{itemize}

\begin{figure*}[t]
\begin{center}
\includegraphics[width=0.99\textwidth, height=10cm, keepaspectratio]{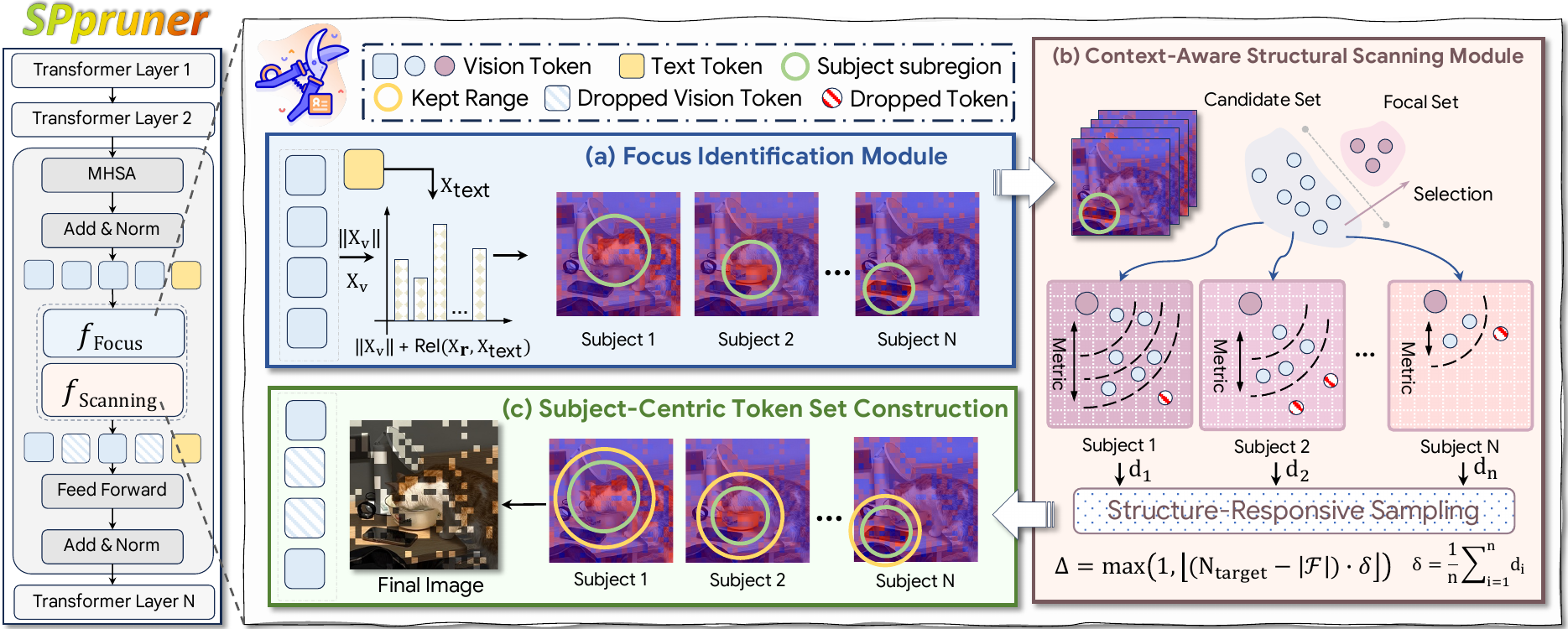}
\end{center}
\caption{\textbf{Framework of SPpruner.} (a) The focus identification module first identifies salient visual subjects by combining intrinsic visual saliency with semantic relevance to the text query. (b) The context-aware structural scanning module then employs a structure-responsive sampling mechanism to select contextual tokens associated with these identified subjects, ensuring structural integrity. (c) Construct the final retained visual token set $\tilde{\mathbf{X}}_{\mathrm{v}}$ by combining the focal subject tokens with their associated contextual tokens.}
\label{fig:framework}\vspace{-2ex}
\end{figure*}

\vspace{-2ex}
\section{Related Work}
\subsection{Vision-Language Models} The integration of large language models with high-capacity vision encoders has catalyzed a paradigm shift in multimodal reasoning~\citep{vqamllms, reasoningmllms}. Pioneering architectures such as LLaVA~\citep{llava} and MiniGPT-4~\cite{minigpt4} align pre-trained visual representations with LLMs via visual instruction tuning, establishing robust baselines for visual understanding. To capture fine-grained details and support long-context inputs, recent advancements like Qwen-VL~\citep{qwen2vl}, mPLUG-Owl2~\cite{mplugowl2}, LLaVA-Next~\citep{llava-next}, LLaVA-Onevision~\citep{llavaonevision}, and InternVL~\cite{internvl} have aggressively scaled up input resolutions and vision encoder capacities. While yielding exceptional performance, these models inevitably produce massive visual token sequences, leading to high computational overhead. This significant computational cost severely hinders their practical deployment in resource-constrained environments, necessitating effective token reduction strategies.

\subsection{Visual Token Reduction}
To mitigate the computational bottleneck, training-free token reduction (TR) has emerged as a promising direction. Existing approaches can be broadly categorized into two paradigms based on their reduction criteria: \textit{Redundancy-Centric Reduction} and \textit{Query-Centric Reduction}.

\textbf{Redundancy-Centric Reduction.} This category of methods operates on the premise that visual data contains inherent spatial or statistical redundancy. They aim to merge or prune tokens by leveraging intrinsic visual properties, treating token reduction as a statistical compression task. For example, ToMe~\citep{tome} uses bipartite matching or clustering based on visual similarity to merge spatially redundant tokens. Similarly, VisionZip~\citep{visionzip} and FastV~\citep{fastv} filter tokens based on vision attention, assuming that tokens with low attention weights in early layers are dispensable. Moreover, DivPrune~\citep{divprune} formulates token reduction as a maximization task, relying on the similarity among visual tokens for reduction. Although these methods effectively reduce the sequence length, they suffer from semantic agnosticism. By relying solely on statistical metrics, they risk merging semantically distinct concepts, leading to a significant degradation in fine-grained and holistic visual perception.

\textbf{Query-Centric Reduction.} Recognizing the importance of user query, recent works have shifted towards query-centric reduction. Methods such as SparseVLM~\citep{sparsevlm}, DART~\citep{dart}, and PACT~\citep{pact} explicitly utilize the cross-attention scores between visual tokens and the textual query to steer token selection. While this improves alignment with the user's query, these approaches exhibit a fundamental limitation: they aggressively discard visual subjects that are not explicitly mentioned in the text but are critical for contextual reasoning, such as contextual cues among the visual subjects. This results in the retained visual tokens failing to maintain consistency with the visual input, severing the global dependencies required for holistic image understanding. This issue will be further addressed in the subsequent sections through both theoretical discussion and empirical validation of our novel token reduction paradigm.

\vspace{-1ex}
\section{Methodology}
\subsection{Preliminaries}
\noindent \textbf{Architecture of Vision-Language Models.}
In this work, we employ a standard vision-language model architecture comprising a vision encoder, a cross-modal projector, and a large language model backbone. The projector maps raw visual inputs into a sequence of continuous visual embeddings. Formally, the input sequence concatenates system instructions $\mathbf{X_s}$, visual tokens $\mathbf{X_v}$, and the user query $\mathbf{X_q}$. During generation phase, LLM decodes the output token $\mathrm{t}$ sequentially based on $\mathbf{X}_{\text{input}}=\{\mathbf{X_s};\mathbf{X_v};\mathbf{X_q}\}$, which can be formulated as: $\mathrm{t_i}=f(\mathbf{X}_{\text{input}},\mathrm{t_{0}},\mathrm{t_{1}},\mathrm{t_{2}},\cdots,\mathrm{t_{i-1}})$.

\noindent \textbf{Visual Token Reduction and Dilemma.}
The objective of visual token reduction is to alleviate the computational overhead during inference. In contrast to the semantically dense text tokens $ \mathbf{X_{\text{text}}} = \left\{ \mathbf{X_s}, \mathbf{X_q} \right\}$, visual tokens exhibit inherent spatial redundancy, rendering them prime candidates for sparsification. We formulate the token reduction problem as follows: Let $\mathbf{X_v}$ be the original set of visual tokens with cardinality $\mathrm{N}$. We seek to identify a sparse subset $\tilde{\mathbf{X}}_\mathbf{v}$ of size $\tilde{\mathrm{N}}$ (where $\tilde{\mathrm{N}} < \mathrm{N}$) that maximally preserves the information required for reasoning. Mathematically, this transforms into finding an optimal subset that minimizes the divergence $\mathcal{L}$ between the predictions of the original and compressed inputs, subject to a budget constraint:
\begin{equation}
\tilde{\mathbf{X}}_\mathrm{v} =
\mathop{\mathrm{argmin}}_{
\tilde{\mathbf{X}}_\mathrm{v} \subset \mathbf{X}_\mathrm{v},\
|\tilde{\mathbf{X}}_\mathbf{v}| = \tilde{\mathrm{N}}
}
\mathcal{L} \Bigl( \mathsf{P}(\mathbf{X}_{\text{text}}; \mathbf{X}_\mathrm{v}),
\tilde{\mathsf{P}}(\mathbf{X}_{\text{text}}; \tilde{\mathbf{X}}_\mathrm{v}) \Bigr).
\label{eq:loss_token_reduction}
\end{equation}
\noindent While existing TR tasks typically leverage unimodal ($\mathbf{X}_\mathrm{v}$) or cross-modal ($\mathbf{X}_\mathrm{text},\mathbf{X}_\mathrm{v}$) information to guide visual token retention, this strict alignment paradigm inevitably leads to the loss of implicit subjects and contextual cues, compromising model's holistic understanding of the entire image.
\subsection{Framework of SPpruner}
In this section, we present a subject-centric progressive token reduction paradigm that excavates a broader salient subject spectrum and more informative global relational dependencies without incurring additional training and fine-tuning, reconciling the conflict between efficient inference and comprehensive visual understanding.
\subsection{Focus Identification Module (FIM)}
The first objective is to mitigate the \textit{Loss of Focus} by identifying visual subjects that are visually salient and semantically relevant. Unlike prior query-guided methods that only
retain tokens strictly aligned with user’s query, we start by establishing a focal set $\mathcal{F}$ to excavate the comprehensive visual subject spectrum to ensure a high-fidelity representation of visual input. Herein, we define a composite scoring function $\mathcal{S}(\mathbf{x}_{\mathrm{i}})$ for each visual token $\mathbf{x}_{\mathrm{i}} \in \mathbf{X}_{\mathrm{v}}$ by synergizing its intrinsic visual magnitude with its semantic relevance:
\begin{equation}
\begin{aligned}
\mathcal{S}(\mathbf{x}_{\mathrm{i}}) &= \Phi\left(\|\mathbf{x}_{\mathrm{i}}\|_1\right) + \Phi\left(\mathcal{R}(\mathbf{x}_{\mathrm{i}} \mid \mathbf{X}_{\mathrm{q}})\right), \\
\mathcal{R}(\mathbf{x}_{\mathrm{i}} \mid \mathbf{X}_{\mathrm{q}}) &= \mathbb{E}_{\mathbf{x}_{\mathrm{q,j}} \in \mathbf{X}_{\mathrm{q}}} \left[ \frac{\mathbf{x}_{\mathrm{i}} \cdot \mathbf{x}_{\mathrm{q,j}}^{\top}}{\|\mathbf{x}_{\mathrm{i}}\|_2 \|\mathbf{x}_{\mathrm{q,j}}\|_2} \right],
\end{aligned}
\label{eq:focus_score}
\end{equation}
\noindent where $\Phi$ denotes a min-max normalization function aligning the metrics to a unified scale. Here, the first term serves as a robust proxy for intrinsic visual saliency, while the latter captures the semantic alignment. By establishing the focal set $\mathcal{F}$ via the top-$\mathrm{K}$ tokens from $\mathcal{S}(\mathbf{x}_{\mathrm{i}})$, SPpruner identifies a broad spectrum of salient subjects, spanning from explicit query targets to implicit but salient subjects.

\begin{table*}[t]
\centering
\caption{Performance comparison on LLaVA-1.5-7B under varying token budgets.}
\label{tab:llava-1.57b}
\renewcommand{\arraystretch}{1}
\resizebox{0.99\textwidth}{!}{
\begin{tabular}{l|ccccccccc|c}
\toprule
Methods
& MME & MMB(EN) & MMB(CN)
& POPE & VizWiz & GQA & VQAv2 & SQA$^\text{IMG}$ & TextVQA & Rel.\\
\midrule

\rowcolor{lightgray}
\multicolumn{11}{c}{\textit{Upper Bound: 576 Tokens (100\%)}} \\

\textcolor{gray}{Vanilla}
& \textcolor{gray}{1862} & \textcolor{gray}{64.6} & \textcolor{gray}{58.1}
& \textcolor{gray}{85.9} & \textcolor{gray}{50.1} & \textcolor{gray}{61.9} & \textcolor{gray}{78.5}
& \textcolor{gray}{69.5} & \textcolor{gray}{58.2} & 100\% \\
\midrule
\rowcolor{lightgray}
\multicolumn{11}{c}{\textit{Retain 192 Tokens} ($\downarrow$ 66.7\%)} \\

FastV {\scriptsize (ECCV'24)}
& 1612 & 61.2 & 57.0 & 64.8 & 50.8 & 52.7 & 67.1 & 67.3 & 52.5 & 90.0\% \\
PyramidDrop {\scriptsize (CVPR'25)}
& 1791 & 63.1 & 57.1 & 85.1 & 51.1 & 59.3 & 75.1 & 68.9 & 56.9 & 98.7\% \\
VisionZip {\scriptsize (CVPR'25)}
& 1769 & 62.9 & 57.4 & 85.5 & 51.6 & 59.3 & 76.8 & 68.8 & 57.2 & 99.0\% \\
DivPrune {\scriptsize (CVPR'25)}
& 1765 & 62.3 & 56.3 & \textbf{87.1} & \textbf{52.8} & 60.1 & 76.0 & 68.7 & 56.6 & 98.8\% \\
DART {\scriptsize (EMNLP'25)}
& 1834 & 63.7 & 57.3 & 82.7 & 52.8 & 60.0 & 74.7 & 68.9 & 57.3 & \underline{99.0}\% \\
PACT {\scriptsize (CVPR'25)}
& 1790 & 63.0 & 56.4 & 83.8 & 52.6 & \textbf{60.7} & 74.8 & 69.0 & 55.5 & 98.1\% \\
\rowcolor{lightblue}
\textbf{SPpruner (Ours)}
& \textbf{1841} & \textbf{63.8} & \textbf{57.6} & 83.3 & 52.7 & 60.5 & \textbf{77.8} & \textbf{69.2} & \textbf{57.5} & \textbf{99.4}\% \\
\midrule

\rowcolor{lightgray}
\multicolumn{11}{c}{\textit{Retain 128 Tokens} ($\downarrow$ 77.8\%)} \\

FastV {\scriptsize (ECCV'24)}
& 1490 & 56.1 & 56.4 & 59.8 & 51.3 & 49.6 & 61.8 & 60.2 & 50.6 & 82.9\% \\
PyramidDrop {\scriptsize (CVPR'25)}
& 1765 & 62.3 & 56.6 & 83.1 & 51.0 & 57.6 & 72.9 & 68.5 & 56.7 & 96.5\% \\
VisionZip {\scriptsize (CVPR'25)}
& 1758 & 62.0 & 56.7 & 83.2 & 52.9 & 57.7 & 75.6 & 68.5 & 56.8 & \underline{97.6}\% \\
DivPrune {\scriptsize (CVPR'25)}
& 1711 & 61.3 & 54.9 & \textbf{87.0} & \textbf{53.6} & 59.3 & 74.1 & 68.4 & 56.0 & 97.4\% \\
DART {\scriptsize (EMNLP'25)}
& 1844 & 63.1 & \textbf{57.3} & 80.1 & 53.5 & 58.2 & 71.3 & 68.7 & 56.6 & 97.2\% \\
PACT {\scriptsize (CVPR'25)}
& 1757 & 62.2 & 54.3 & 83.0 & 52.8 & \textbf{60.5} & 74.8 & 69.1 & 54.9 & 96.9\% \\
\rowcolor{lightblue}
\textbf{SPpruner (Ours)}
& \textbf{1849} & \textbf{63.4} & 57.0 & 80.3 & 53.0 & 59.2 & \textbf{77.8} & \textbf{69.2} & \textbf{56.9} & \textbf{98.5}\% \\
\midrule

\rowcolor{lightgray}
\multicolumn{11}{c}{\textit{Retain 64 Tokens} ($\downarrow$ 88.9\%)} \\ 
FastV {\scriptsize (ECCV'24)}
& 1256 & 48.0 & 52.7 & 48.0 & 50.8 & 46.1 & 55.0 & 51.1 & 47.8 & 74.0\% \\
PyramidDrop {\scriptsize (CVPR'25)}
& 1579 & 57.9 & 49.5 & 74.4 & 50.7 & 54.6 & 69.2 & 69.2 & 53.7 & 88.6\% \\
VisionZip {\scriptsize (CVPR'25)}
& 1688 & 55.1 & 45.9 & 77.0 & 52.4 & 55.1 & 72.4 & 69.0 & \textbf{55.6} & 90.4\% \\
DivPrune {\scriptsize (CVPR'25)}
& 1615 & 60.0 & 52.5 & \textbf{85.6} & 53.3 & \textbf{57.7} & 71.2 & 67.9 & 54.6 & \underline{94.4}\% \\
DART {\scriptsize (EMNLP'25)}
& 1749 & 61.4 & 53.9 & 73.8 & 52.5 & 56.0 & 67.1 & 68.6 & 54.1 & 93.4\% \\
PACT {\scriptsize (CVPR'25)}
& 1564 & 60.1 & 52.1 & 81.7 & 52.2 & 59.5 & 73.1 & 69.0 & 53.2 & 94.0\% \\
\rowcolor{lightblue}
\textbf{SPpruner (Ours)}
& \textbf{1750} & \textbf{62.0} & \textbf{54.6} & 75.2 & \textbf{52.7} & 56.3 & \textbf{73.6} & \textbf{69.6} & 55.4 & \textbf{95.2}\% \\
\bottomrule
\end{tabular}} \vspace{-2ex}
\end{table*}

\subsection{Context-Aware Structural Scanning Module (CASSM)}
While the identified focal tokens capture the core subjects, a disjoint set of visual subject tokens is insufficient for holistic image understanding. To address the \textit{Loss of Context}, we develop a context-aware structural scanning module that simulates the scanning mechanism of human visual system to aggregate contextual cues and restore global relationships associated with these subjects. To ensure the preservation of sufficient contextual cues, we define a contextual utility function $\mathcal{U}(\mathbf{x}^{\mathrm{c}})$ that evaluates candidate tokens $\mathbf{x}^{\mathrm{c}}$ across two dimensions: \textit{structural dependency} and \textit{semantic alignment}, collectively yielding a informative representation of the subject's neighboring structural context:
\begin{equation}
\mathcal{U}(\mathbf{x}^{\mathrm{c}}) =
\underbrace{\mathcal{M}(\mathbf{x}^{\mathrm{c}} \mid \mathcal{F}, \mathcal{C})}_{\text{Structural Dependency}} +
\underbrace{\mathcal{R}(\mathbf{x}^{\mathrm{c}} \mid \mathbf{X}{\mathrm{q}})}_{\text{Semantic Alignment}},
\label{eq:UtilityFunctional}
\end{equation}
\noindent where $\mathcal{C}$ is candidate token set, $\mathcal{M}(\cdot)$ captures structural dependency between subject and candidate tokens, and $\mathcal{R}(\cdot)$ ensures that the scanning remains relevant to user's query. This formulation ensures that preserved tokens capture sufficient structural context, while remaining highly aligned with the user's query. Furthermore, the structural dependency term $\mathcal{M}(\mathbf{x}^{\mathrm{c}} \mid \mathcal{F}, \mathcal{C})$ is evaluated from two perspectives:
\begin{equation}
\begin{aligned}
    \mathcal{M}(\mathbf{x}^{\mathrm{c}} \mid \mathcal{F},\mathcal{C})=\mathbb{E}_{\mathbf{x}^{\mathrm{f}} \in \mathcal{F}}
    \left[\frac{\mathbf{x}^{\mathrm{c}} \cdot \mathbf{x}^{{\mathrm{f}}\top}}{\|\mathbf{x}^{\mathrm{c}}\|_2 , \|\mathbf{x}^{\mathrm{f}}\|_2}
    \right]+ \\
    \mathbb{E}_{\mathbf{z} \in \mathcal{C}}
    \left[1 - \frac{\mathbf{x}^{\mathrm{c}} \cdot \mathbf{z}^{\top}}{\|\mathbf{x}^{\mathrm{c}}\|_2 , \|\mathbf{z}\|_2}
    \right],
\end{aligned}
\label{eq:Operator}
\end{equation}
\noindent where the first term encourages focal tokens to maintain appropriate structural correlation with their neighboring tokens, preserving local structural coherence. The second term evaluates global discriminability among candidate tokens, facilitating a more discriminative structural representation.

Consequently, we develop a \textit{Structure-Responsive Sampling} (SRS) mechanism to optimize scanning granularity. This module progressively modulates retention stride according to structural divergence between focal tokens and candidate tokens, ensuring the retention of the most informative contextual cues. This procedure can be formulated as:
\begin{equation}
\begin{aligned}
    \mathbf{d}(\mathbf{x}^{\mathrm{c}}) &= \exp\left( -\mathbb{E}_{\mathbf{x}^{\mathrm{f}} \in \mathcal{F}}
    \left[\frac{\mathbf{x}^{\mathrm{c}} \cdot \mathbf{x}^{{\mathrm{f}}\top}}{\|\mathbf{x}^{\mathrm{c}}\|_2 , \|\mathbf{x}^{\mathrm{f}}\|_2}
    \right] \right), \\ \delta &= \mathbb{E}_{\mathbf{c} \in \mathcal{C}} [\mathbf{d}(\mathbf{x}^{\mathrm{c}})].
\end{aligned}
\end{equation}
The retention stride $\Delta$ is then dynamically adjusted:
\begin{equation}
\Delta = \max \left( 1, \lfloor (\mathrm{N}_{\text{target}} - |\mathcal{F}|) \cdot \delta \rfloor \right).
\end{equation}
When candidate tokens lie in proximity to candidate tokens, the retention stride $\Delta$ will increase to preserve more contextual information reflecting the subject's content. Conversely, the stride should be reduced. Intuitively, this strategy adaptively adjusts the sampling stride based on the structural discrepancy between the subject and candidate tokens, reconciling contextual sufficiency and compression efficiency.

\subsection{Theoretical Analysis}
To rigorously justify the reliability of SPpruner, we analyze the theoretical upper bound of the approximation error induced by our token reduction paradigm. Let $f(\cdot)$ denote the transformer function and $\mathbf{X}_{\mathrm{v}}$ be the input visual tokens. Our goal is to bound the error $\|f(\mathbf{X}_{\mathrm{v}}) - f(\tilde{\mathbf{X}}_\mathbf{v})\|$, where $\tilde{\mathbf{X}}_\mathbf{v}$ is retained token set. Based on the Lipschitz continuity assumption of Transformers (Assumption~\ref{ass:transformer_property}), the output error is bounded by the Hausdorff distance $\mathbf{d}_\mathrm{H}(\mathbf{X}_{\mathrm{v}}, \tilde{\mathbf{X}}_\mathbf{v})$ between the original and reduced token sets. We formally present the theoretical guarantee of SPpruner as follows:

\begin{theorem}
\label{thm:upper_bound}
(Approximation Error Bound). Given initial visual token set $\mathbf{X}_{\mathrm{v}}$ and retained subset $\tilde{\mathbf{X}}_\mathbf{v}$ retained by SPpruner. Under the assumption that the transformer $f$ is $K$-Lipschitz continuous with respect to the Hausdorff distance, the approximation error is upper-bounded by:
\begin{equation}
\|f(\mathbf{X}_{\mathrm{v}}) - f(\tilde{\mathbf{X}}_\mathbf{v})\| \leq K \cdot \max_{\mathbf{x} \in \mathbf{X}_{\mathrm{v}} \setminus \tilde{\mathbf{X}}_\mathbf{v}} \min_{\tilde{\mathbf{x}} \in \tilde{\mathbf{X}}_\mathbf{v}} \|\mathbf{x} - \tilde{\mathbf{x}} \|.
\end{equation}
\end{theorem}
Crucially, SPpruner minimizes this bound through the dual-module design: focus identification module ensures that high-fidelity subject tokens are included in $\tilde{\mathbf{X}}_{\mathrm{v}}$, while the context-aware structural scanning module employs a structural-responsive sampling mechanism to progressively adjust retention sampling stride $\Delta$ based on structural divergence. This ensures that any discarded token $\mathbf{x}$ is theoretically situated within a specific radius of a retained token $\tilde{\mathbf{x}}$, thereby controlling the worst-case structural deviation. This theoretical result confirms that SPpruner can preserve the theoretical visual understanding capability of the vanilla model, and a detailed proof is provided in the Appendix~\ref{seq:proof_of_theorem}.

\begin{table*}[t]
\centering
\caption{Performance comparison on LLaVA-Next-7B under varying token budgets.}
\label{tab:llava-next}
\renewcommand{\arraystretch}{1}
\resizebox{0.99\textwidth}{!}{
\begin{tabular}{l|cccccccc|c}
\toprule
Methods
& MME & MMB(EN) & MMB(CN)
& POPE & MMVet & GQA 
& SQA$^\text{IMG}$ & TextVQA &  Rel.\\
\midrule

\rowcolor{lightgray}
\multicolumn{10}{c}{\textit{Upper Bound: 2880 Tokens (100\%)}} \\

\textcolor{gray}{Vanilla}
& \textcolor{gray}{1842} & \textcolor{gray}{65.8} & \textcolor{gray}{57.3}
& \textcolor{gray}{86.8} & \textcolor{gray}{40.0} & \textcolor{gray}{62.5}
& \textcolor{gray}{67.6} & \textcolor{gray}{60.3} & 100\% \\
\midrule

\rowcolor{lightgray}
\multicolumn{10}{c}{\textit{Retain 320 Tokens} ($\downarrow$ 88.9\%)} \\

FastV {\scriptsize (ECCV'24)}
& 1661 & 61.6 & 51.9 & 71.7 & 25.0 & 55.9 & 62.8 & 55.7 & 86.8\% \\
PyramidDrop {\scriptsize (CVPR'25)}
& 1660 & 57.9 & 49.5 & 81.2 & 30.2 & 56.4 & 69.5 & 56.9 & 90.1\% \\
VisionZip {\scriptsize (CVPR'25)}
& 1715 & 63.1 & 55.7 & 82.2 & 38.8 & 58.9 & 67.6 & 58.8 & 96.2\% \\
DivPrune {\scriptsize (CVPR'25)}
& 1702 & 64.2 & 55.7 & \textbf{84.7} & 38.9 & 61.1 & 67.7 & 56.2 & 96.1\% \\
DART {\scriptsize (EMNLP'25)}
& 1697 & 65.4 & 56.7 & 83.5 & \textbf{40.6} & 59.4 & 68.3 & 58.4 & \underline{96.4}\% \\

\rowcolor{lightblue}
\textbf{SPpruner (Ours)}
& \textbf{1730} & \textbf{65.0} & \textbf{57.5}
& 84.0 & 38.1 & \textbf{60.8}
& \textbf{68.0} & \textbf{59.3} & \textbf{97.7}\% \\
\bottomrule
\end{tabular}}
\end{table*}

\begin{table*}[t]
\centering
\caption{Performance comparison on Qwen2.5-VL-32B under varying token budgets.}
\label{tab:qwen2.5vl-32b}
\renewcommand{\arraystretch}{1.05}
\setlength{\tabcolsep}{4pt}
\resizebox{0.99\textwidth}{!}{
\begin{tabular}{l|cccccccccccc|c}
\toprule
Methods
& MMB(EN) & MMB(CN) & MME-P & MME-C & MMStar & POPE & MMMUVal
& MuirBench & RealWorld & TextVQA & DocVQA & OCRBench & Rel. \\
\midrule

\rowcolor{lightgray}
\multicolumn{14}{c}{\textit{Upper Bound: All Tokens (100\% Tokens)}} \\
\midrule

\textcolor{gray}{Vanilla}
& \textcolor{gray}{86.2} & \textcolor{gray}{83.3} & \textcolor{gray}{1731}
& \textcolor{gray}{704}  & \textcolor{gray}{65.5} & \textcolor{gray}{84.2}
& \textcolor{gray}{36.2} & \textcolor{gray}{47.3} & \textcolor{gray}{68.5} & \textcolor{gray}{76.5}
& \textcolor{gray}{92.7} & \textcolor{gray}{822} & 100\% \\
\midrule

\rowcolor{lightgray}
\multicolumn{14}{c}{\textit{Retain 70\% Tokens} ($\downarrow$ 30\%)} \\

FastV {\scriptsize (ECCV'24)}
&85.7 &\textbf{82.8} &1674 &\textbf{718} &63.7 &\textbf{84.9} &37.6 &47.4 &66.5 &68.4 &84.4 &711 & 97.0\% \\
DART {\scriptsize (EMNLP'25)}
&86.0 &82.6 &\textbf{1724} &702 &63.8 &83.2 &36.6 &47.1 &\textbf{68.6} &75.9 &92.2 &805 & 99.3\% \\
\rowcolor{lightblue}
\textbf{SPpruner (Ours)}
&\textbf{86.1} &82.5 &1712 &709 &\textbf{64.9} &84.0 &34.8 &47.4 &68.4 &\textbf{76.0} &\textbf{92.4} &\textbf{814} & \textbf{99.3\%} \\
\midrule

\rowcolor{lightgray}
\multicolumn{14}{c}{\textit{Retain 50\% Tokens} ($\downarrow$ 50\%)} \\

FastV {\scriptsize (ECCV'24)}
&85.1 &81.2 &1628 &\textbf{698} &60.0 &82.7 &36.0 &47.0 &67.4 &67.8 &78.5 &632 & 93.9\% \\
DART {\scriptsize (EMNLP'25)}
&85.3 &81.6 &1696 &664 &61.9 &83.3 &35.9 &46.8 &67.3 &74.6 &87.2 &751 & 96.9\% \\
\rowcolor{lightblue}
\textbf{SPpruner (Ours)}
&\textbf{85.6} &\textbf{82.0} &\textbf{1711} &689 &\textbf{63.8} &\textbf{83.5} &\textbf{37.2} &\textbf{47.0} &\textbf{68.0} &\textbf{75.3} &\textbf{91.1} &\textbf{798} & \textbf{98.9\%} \\
\midrule

\rowcolor{lightgray}
\multicolumn{14}{c}{\textit{Retain 30\% Tokens} ($\downarrow$ 70\%)} \\

FastV {\scriptsize (ECCV'24)}
&80.8 &76.2 &1523 &580 &52.0 &76.0 &31.6 &46.8 &64.7 &65.3 &63.5 &510 & 86.8\% \\
DART {\scriptsize (EMNLP'25)}
&82.5 &79.5 &1654 &625 &56.6 &80.3 &35.5 &\textbf{47.0} &63.9 &72.0 &69.4 &628 & 92.0\% \\
\rowcolor{lightblue}
\textbf{SPpruner (Ours)}
&\textbf{84.5} &\textbf{79.5} &\textbf{1662} &\textbf{657} &\textbf{58.6} &\textbf{81.0} &\textbf{35.9} &46.8 &\textbf{66.8} &\textbf{73.0} &\textbf{81.0} &\textbf{704} & \textbf{94.4\%} \\

\bottomrule
\end{tabular}} \vspace{-1ex}
\end{table*}

\begin{table*}[t]
\centering
\caption{\textbf{Efficiency Comparison on Qwen2.5-VL.} We compare SPpruner against the Vanilla baseline and DART~\cite{dart}. \textbf{Latency} refers to the prefill time. SPpruner achieves the highest speedup ratios with negligible performance impact across all benchmarks.}
\label{tab:efficiency_analysis}
\renewcommand{\arraystretch}{1}
\resizebox{0.99\textwidth}{!}{
\begin{tabular}{l|l|c|ccccc}
\toprule
Benchmark & Method & Visual Tokens & Prefill Time (s) $\downarrow$ & FLOPs (T) $\downarrow$ & GPU Mem. (GB) $\downarrow$ & Performance $\uparrow$ & SpeedUp $\uparrow$\\
\midrule

\multirow{3}{*}{MME}
 & Vanilla & 866 & 0.2215 & 2.1015 & 16.01 & 2329 & 1.00$\times$ \\
 & DART {\scriptsize (CVPR'25)}    & 192 & 0.1068 & 0.5798 & 15.86 & 2174 & 2.07$\times$ \\

\rowcolor{lightblue} \cellcolor{white}
 & \textbf{SPpruner {\scriptsize (Ours)} }   & 192 & \textbf{0.0876} & \textbf{0.5798} & \textbf{15.86} & \textbf{2182} & \textbf{2.53}$\times$ \\
\midrule

\multirow{3}{*}{MMStar}
 & Vanilla & 360 & 0.1337  & 0.9802 & 15.97 & 62.9 & 1.00$\times$ \\
 & DART {\scriptsize (CVPR'25)} & 80 & 0.0833 & 0.3688 & 15.74 & 49.7 & 1.60$\times$ \\
\rowcolor{lightblue} \cellcolor{white}
 & \textbf{SPpruner {\scriptsize (Ours)} }   & 80 & \textbf{0.0618} & \textbf{0.3688} & \textbf{15.74} & \textbf{50.7} & \textbf{2.16$\times$} \\
\midrule

\multirow{3}{*}{MMB(EN)}
 & Vanilla & 278 & 0.1086 & 0.8232 & 15.99 & 83.9 & 1.00$\times$ \\
 & DART {\scriptsize (CVPR'25)}      & 62 & 0.0752 & 0.3537 & 15.60    & 78.2 & 1.44$\times$ \\
\rowcolor{lightblue} \cellcolor{white}
 & \textbf{SPpruner {\scriptsize (Ours)} }   & 62 & \textbf{0.0615} & \textbf{0.3537} & \textbf{15.60}    & \textbf{80.0} & \textbf{1.77}$\times$ \\

\bottomrule
\end{tabular}}\vspace{-2ex}
\end{table*}

\begin{table}[t]
\centering
\caption{Performance comparison of Qwen2.5-VL-32B on Flickr30K under varying token budgets.}
\label{tab:qwen2.5vl-32bcaption}
\renewcommand{\arraystretch}{1.2}
\setlength{\tabcolsep}{4pt}
\LARGE
\resizebox{\columnwidth}{!}{
\begin{tabular}{l|ccccccc|c}
\toprule
Methods
& BLEU-1 & BLEU-2 & BLEU-3 & BLEU-4 & METEOR & ROUGE-L & CIDEr & Rel. \\
\midrule

\rowcolor{lightgray}
\multicolumn{9}{c}{\textit{Upper Bound: All Tokens (100\% Tokens)}} \\

\textcolor{gray}{Vanilla}
& \textcolor{gray}{0.4653} & \textcolor{gray}{0.2791} & \textcolor{gray}{0.1694}
& \textcolor{gray}{0.1027} & \textcolor{gray}{0.1963} & \textcolor{gray}{0.3652}
& \textcolor{gray}{0.2903} & 100\% \\
\midrule



\rowcolor{lightgray}
\multicolumn{9}{c}{\textit{Retain 50\% Tokens} ($\downarrow$ 50\%)} \\


FastV {\scriptsize (ECCV'24)}
&0.4625 &0.2750 &0.1652 &0.0988 &0.1900 &0.3603 &0.2947 & 98.4\% \\
DART {\scriptsize (EMNLP25)}
&0.4605 &0.2746 &0.1660 &0.1001 &0.1933 &0.3608 &0.2854 & 98.3\% \\
\rowcolor{lightblue}
\textbf{SPpruner {\scriptsize (Ours)}}
&\textbf{0.4628} &\textbf{0.2769} &\textbf{0.1680} &\textbf{0.1017} &\textbf{0.1945} &\textbf{0.3626} &\textbf{0.2894} & \textbf{99.3}\% \\
\midrule

\rowcolor{lightgray}
\multicolumn{9}{c}{\textit{Retain 30\% Tokens} ($\downarrow$ 70\%)} \\


FastV {\scriptsize (ECCV'24)}
&0.4436 &0.2561 &0.1496 &0.0870 &0.1763 &0.3432 &0.2647 & 90.7\% \\
DART {\scriptsize (EMNLP25)}
&0.4524 &0.2659 &0.1588 &0.0946 &0.1864 &0.3519 &0.2764 & 95.0\% \\
\rowcolor{lightblue}
\textbf{SPpruner {\scriptsize (Ours)}}
&\textbf{0.4591} &\textbf{0.2723} &\textbf{0.1645} &\textbf{0.0990} &\textbf{0.1903} &\textbf{0.3572} &\textbf{0.2884} & \textbf{97.7}\% \\
\midrule

\rowcolor{lightgray}
\multicolumn{9}{c}{\textit{Retain 10\% Tokens} ($\downarrow$ 90\%)} \\

FastV {\scriptsize (ECCV'24)}
&0.3773 &0.1917 &0.1016 &0.0546 &0.1345 &0.2927 &0.1425 & 65.8\% \\
DART {\scriptsize (EMNLP25)}
&0.4156 &0.2289 &0.1293 &0.0736 &0.1603 &0.3169 &0.2250 & 80.8\% \\
\rowcolor{lightblue}
\textbf{SPpruner {\scriptsize (Ours)}}
&\textbf{0.4245} &\textbf{0.2376} &\textbf{0.1358} &\textbf{0.0779} &\textbf{0.1658} &\textbf{0.3234} &\textbf{0.2457} & \textbf{84.3}\% \\

\bottomrule
\end{tabular}}
\end{table}
\vspace{-1ex}
\section{Experiments}
In this section, we evaluate SPpruner on image understanding task, document understanding task, image captioning task, and video understanding task on a total of 22 benchmarks, comparing its performance with previous token reduction paradigms. Furthermore, we also provide the ablation studies for the metrics and hyperparameters used in SPpruner. Extensive experiments demonstrate that SPpruner has superior performance in preserving a broader visual subject spectrum and structural context information, ensuring the model retains strong visual understanding capabilities in complex scenes even after token reduction. Details of the benchmarks and model are provided in the Appendix~\ref{seq::detailed_experiment_settings}.

\subsection{Image Understanding Task}
For both single-image and multi-image benchmarks, we conducted extensive evaluations on LLaVA-1.5-7B, LLaVA-Next-7B, and Qwen2.5VL-32B. As shown in Table~\ref{tab:llava-1.57b}, Table~\ref{tab:llava-next}, and Table~\ref{tab:qwen2.5vl-32b}, SPpruner maintains its performance advantage and achieves SOTA results across evaluations. Taking LLaVA-Next-7B as an instance, SPpruner surpasses FastV, Divprune, and DART by 3.7\%, 1.5\%, and 1.8\%, respectively, on the MME benchmark at 88.9\% reduction ratio. Notably, in the multi-image evaluation of Qwen2.5-VL-32B at 70\% reduction ratio, SPpruner outperforms FastV and DART algorithms by 1.6\% and 5.0\%, respectively.


\subsection{Document Understanding Task}
Document comprehension tasks impose stringent demands on localized details and spatial layouts in
images, presenting challenges for compression algorithms. To assess the efficacy of SPpruner, we conduct extensive evaluations, using the LLaVA series and Qwen2.5-VL on three challenging document understanding benchmarks: TextVQA, DocVQA, and OCRBench. By preserving sufficient contextual cues, SPpruner maintains a significant performance advantage even under high reduction ratios. Notably, as illustrated in Table~\ref{tab:qwen2.5vl-32b}, SPpruner shows an average and great 3.6\% improvement over sub-optimal Dart on three document understanding benchmarks under 50\% reduction ratio, validating the efficacy of our approach in retaining fine-grained contextual details. Additionally, SPpruner exhibits stronger capability in maintaining performance compared to competing algorithms, particularly in high-ratio reduction scenarios.


\subsection{Image Captioning Task}
Image captioning serves as a critical measurement for evaluating the preservation of key visual subjects and their contextual information. To assess the generalization capability of SPpruner, we extended our evaluation to Flickr, and results on Qwen2.5-VL-32B are detailed in Table~\ref{tab:qwen2.5vl-32bcaption}. While FastV, DART, and SPpruner yield comparable performance at 50\% reduction ratio, their trajectories diverge significantly as sparsification intensifies. At 70\% reduction, FastV and DART suffer a precipitous performance collapse, dropping by 7.7\% and 3.3\%, respectively. In contrast, SPpruner demonstrates remarkable resilience, retaining 97.7\% of the vanilla performance. Furthermore, SPpruner limits the degradation in average performance to just 2.3\%, outperforming the 5\% decline observed in DART. These empirical findings demonstrate that existing training-free TR algorithms fail to maintain and sustain high fidelity and semantic consistency under high reduction ratios. SPpruner effectively mitigates this by preserving a broader visual subject spectrum and capturing informative structural context associated with these visual subjects, ensuring holistic visual understanding and coherent caption generation capacity.

\begin{table}[t]
\centering
\caption{Video performance comparison on LLaVA-OneVision-7B under varying token budgets.}
\label{tab:llava-onevision}
\renewcommand{\arraystretch}{1.05}
\resizebox{\columnwidth}{!}{
\begin{tabular}{l|ccccc}
\toprule
Methods & VideoMME & $\text{VideoMME}^{\text{w/o}}$ 
& SeedBench & MLVU & Rel. \\
\midrule

\rowcolor{lightgray}
\multicolumn{6}{c}{\textit{Upper Bound: All Tokens (100\% Tokens)}} \\
\midrule

\textcolor{gray}{Vanilla} & \textcolor{gray}{58.4} & \textcolor{gray}{61.2} & \textcolor{gray}{57.0} & \textcolor{gray}{64.7} & 100\% \\
\midrule

\rowcolor{lightgray}
\multicolumn{6}{c}{\textit{Retain 35\% Tokens} ($\downarrow$ 65\%)} \\
\midrule

FastV { (ECCV'24)} & 57.1 & 59.8 & 56.8 & 61.5 & 97.6\% \\
\rowcolor{lightblue}
\textbf{SPpruner (Ours)} & \textbf{57.9} & \textbf{60.2} & \textbf{56.5} & \textbf{64.1} & \textbf{98.9}\% \\
\midrule

\rowcolor{lightgray}
\multicolumn{6}{c}{\textit{Retain 25\% Tokens} ($\downarrow$ 75\%)} \\
\midrule

FastV { (ECCV'24)} &56.0 &58.7 &56.4 & 61.0 & 96.3\%\\
\rowcolor{lightblue}
\textbf{SPpruner (Ours)} &\textbf{57.2} &\textbf{60.4} &\textbf{56.7} & \textbf{63.8} & \textbf{98.7}\%\\
\midrule

\rowcolor{lightgray}
\multicolumn{6}{c}{\textit{Retain 15\% Tokens} ($\downarrow$ 85\%)} \\
\midrule

FastV { (ECCV'24)} &53.6 &57.6 &55.7 & 59.1 & 93.7\%\\
\rowcolor{lightblue}
\textbf{SPpruner (Ours)} &\textbf{56.2} &\textbf{58.7} &\textbf{55.7} & \textbf{62.7} & \textbf{96.7}\%\\

\bottomrule
\end{tabular}}\vspace{-3ex}
\end{table}

\subsection{Video Understanding Task}
We conducted a comparative evaluation of SPpruner and Fastv based on the LLaVA-OneVision. As shown in Table~\ref{tab:llava-onevision}, SPpruner demonstrates superior video understanding performance compared
to FastV. Specifically, SPpruner achieves lossless performance across all video understanding benchmarks with different token budgets. For instance, at 65\% reduction ratio, it incurs only a marginal accuracy degradation of approximately 1.1\%, substantially outperforming FastV. Furthermore, SPpruner exhibits average video understanding metrics that surpass FastV by 2.4\% under 75\% reduction ratio. These results demonstrate that as the reduction ratio increases, SPpruner maintains superior performance without suffering a significant degradation.

\subsection{Efficiency Analysis}
As shown in Table~\ref{tab:efficiency_analysis}, our method significantly enhances inference efficiency. Benchmarked with the Qwen2.5-VL-7B on three image understanding under 77.8\% reduction ratio, SPpruner reduces the prefill latency by 53\% and FLOPs by 64\%, achieving an average speedup of 2.15$\times$. 


\begin{figure*}[t]
\vspace{-1ex}
\begin{center}
\includegraphics[width=0.99\textwidth]{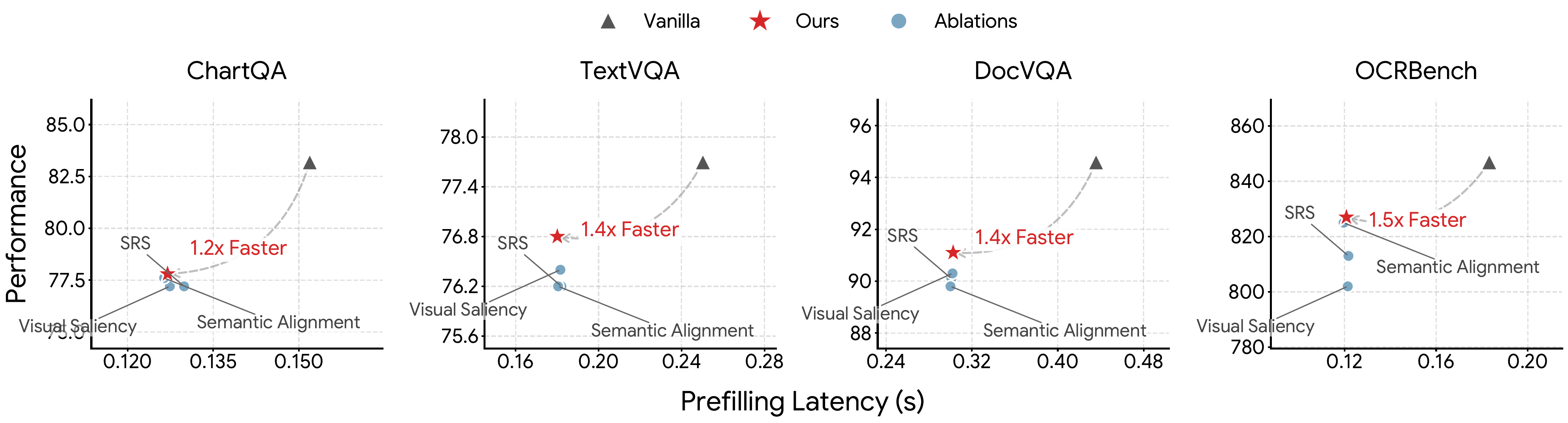}
\end{center}
\vspace{-2ex}
\caption{\textbf{Ablation Studies}. The performance drop without SRS confirms the necessity of adaptive retention strides, while the other metrics validate their role in saliency identification. By unifying these, SPpruner outperforms all variants to achieve 1.2$\times$--1.5$\times$ speedups with comparable accuracy on chart and document understanding tasks.}
\label{fig:ablation_component}
\vspace{-2ex}
\end{figure*}

\begin{figure*}[t]
\begin{center}
\includegraphics[width=0.99\textwidth]{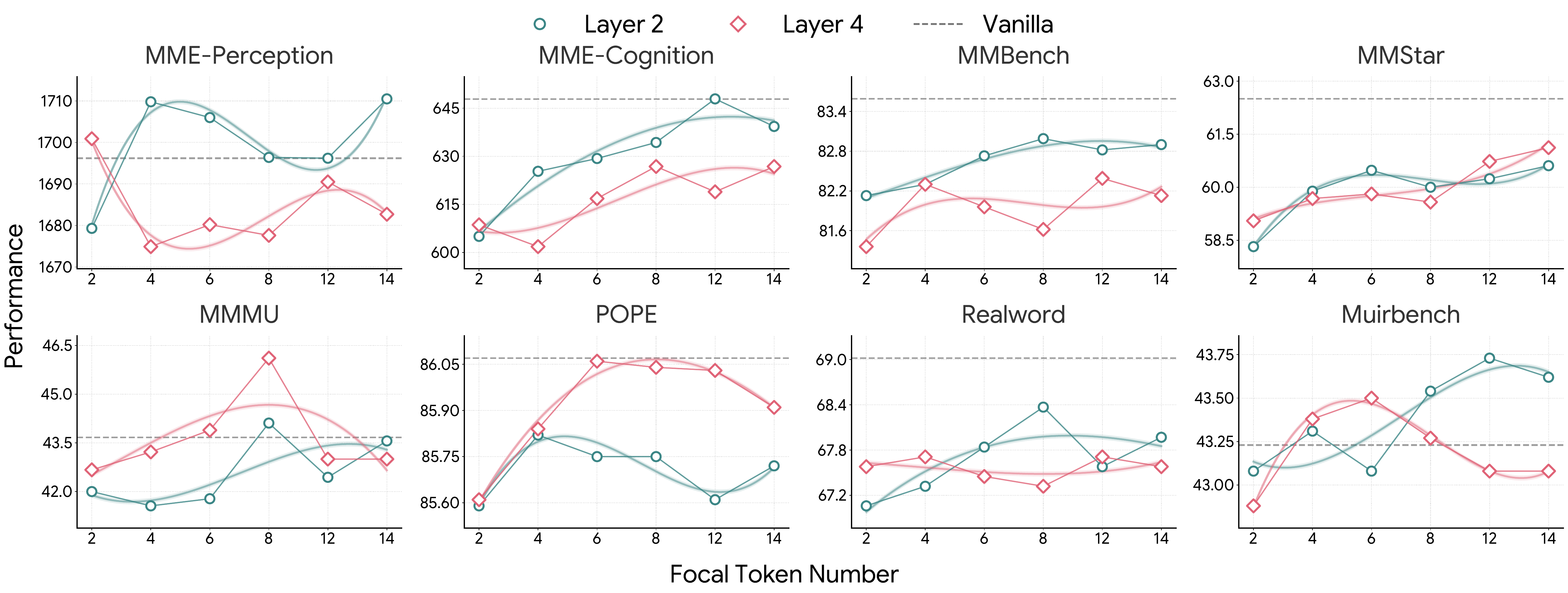}
\end{center}
\vspace{-1ex}
\caption{\textbf{Ablation Studies}. This figure shows that too few focal tokens impair holistic perception by omitting subjects, while too many reduce SRS to generic Top-$\mathrm{K}$ selection due to diminished contextual cues.}
\label{fig:ablation_hyper}
\vspace{-3ex}
\end{figure*}

\subsection{Ablation Studies}
In this section, we conduct ablation studies on the Qwen2.5-VL-7B model across four chart and document understanding benchmarks under 40\% reduction ratio. Furthermore, we provide an in-depth analysis of how the choice of retention metrics and components in our two modules jointly affects both model performance and inference latency.
\paragraph{The importance of the FIM metric and SRS mechanism.} We analyze the contribution of the two metrics in FIM as well as the SRS mechanism during the scanning stage, as shown in Figure~\ref{fig:ablation_component}. Removing either visual saliency or semantic relevance consistently leads to performance degradation across all evaluated benchmarks, indicating that both cues are essential for reliable subject identification. Replacing SRS with a uniform Top-$\mathrm{K}$ selection strategy results in inferior performance under comparable prefilling latency, demonstrating that static sampling fails to effectively balance efficiency and contextual fidelity. In contrast, the proposed SRS mechanism achieves a more favorable trade-off, yielding approximately 1.2$\times$--1.5$\times$ speedup while maintaining comparable performance across chart and document understanding tasks. Notably, the additional overhead introduced by SRS is negligible, as SPpruner progressively adjusts sampling stride based on the structural divergence between focal tokens and candidate tokens, enabling fast and accurate context completion without sacrificing efficiency. \vspace{-2ex}
\paragraph{The selection of focal number and reduction layer.} 
Figure~\ref{fig:ablation_hyper} illustrates the impact of the number of focal tokens and the reduction layer on model performance. When the layer is fixed, increasing the number of focal tokens consistently yields an initial performance gain followed by a decline across all benchmarks. Regarding the number of focal tokens, an insufficient number fails to capture adequate visual subjects, impairing the model’s holistic image understanding. Conversely, when the focal token number is too large, the SRS mechanism degenerates into a simple Top-$\mathrm{K}$ selection, as there is insufficient structural information to guide the retention process. Based on this, we identify 8 as the optimal focal token number, and thus set this value for Qwen2.5-VL. Furthermore, extensive empirical evaluation shows that layer 2 offers the best efficiency–accuracy trade-off, and we thus adopt it as the reduction layer.

\begin{table*}[t]
\centering
\caption{\textbf{Performance Comparison on OpenPangu-Embedded-7B.} We evaluate the performance under various token budgets on the GSM8K benchmark. SPpruner achieves significant prefill time reduction while maintaining high relative performance.}
\label{tab:openpangu-7b}
\renewcommand{\arraystretch}{1}
\resizebox{0.8\textwidth}{!}{
\begin{tabular}{l|l|cccc}
\toprule
Benchmark & Method & Performance $\uparrow$ & Rel. $\uparrow$ & Prefill Time (s) $\downarrow$ & SpeedUp $\uparrow$ \\
\midrule

\multirow{5}{*}{GSM8K}
 & Vanilla (OpenPangu) & 0.8197 & 100\% & 0.273 & 1.00$\times$ \\
 & SPpruner ($\downarrow$84.0\%) & 0.8053 & 98.2\% & 0.217 & 1.26$\times$ \\
 & SPpruner ($\downarrow$88.0\%) & 0.8061 & 98.3\% & 0.217 & 1.26$\times$ \\
 & SPpruner ($\downarrow$91.6\%) & 0.7629 & 93.1\% & 0.204 & 1.34$\times$ \\
 & SPpruner ($\downarrow$92.0\%) & 0.7568 & 92.3\% & 0.203 & 1.34$\times$ \\
\bottomrule
\end{tabular}}\vspace{-3ex}
\end{table*}

\subsection{Extended Experiments}
To further validate the generality and versatility of our proposed subject-centric progressive paradigm, we extend its application beyond Vision-Language Models to pure language models. While the core idea of SPpruner remains modality-agnostic, the inherent differences in information redundancy between linguistic and visual modalities necessitate certain adaptations. Specifically, we remove the visual-text correlation constraints and introduce a hierarchical pruning strategy with varying pruning rates across different layers to better suit textual redundancy. The remainder of the architecture remains unchanged. Based on this adapted design, we conducted additional experiments on the GSM8K \citep{gsm8k} benchmark using OpenPangu-Embedded-7B \citep{openpangu}. As shown in Table 8, our method achieves significant efficiency gains while preserving robust performance: it maintains 92.3\% of the original accuracy even when prefill time is reduced by 25.6\%, and incurs a marginal accuracy drop of merely 1.7\% under a 20.5\% prefill time reduction.

\section{Conclusion}
In this work, we presented SPpruner, a novel subject-centric progressive reduction paradigm designed to overcome the limitations of query-centric reduction. By mirroring the \textit{Focus-then-Context} mechanism of human visual perception, we transform token reduction from a local statistical filtering task into a bio-inspired scanning process. Our approach is built on two properties critical for maintaining visual fidelity: (i) capturing a comprehensive spectrum of visual subjects, and (ii) preserving their structural context. To fulfill these objectives, we first implemented a focus identification module that leverages both intrinsic visual saliency and semantic dependency to capture broad and diverse visual subjects. Subsequently, we developed a context-aware structural scanning module that progressively modulates sampling stride based on structural disparity, precisely restoring the environmental context associated with these subjects. Extensive validation across diverse benchmarks confirms that SPpruner achieves a superior trade-off between inference efficiency and model accuracy. Our architecture-agnostic, plug-and-play solution effectively preserves visual understanding capability in resource-constrained environments.


\section*{Impact Statement}
This paper aims to advance efficient inference in vision-language models and falls under foundational methodological exploration. While our work may yield positive societal impacts regarding efficient inference and the deployment of vision-language systems, none of which we deem necessary to specifically highlight here.

\bibliography{example_paper}
\bibliographystyle{icml2026}

\newpage
\appendix
\onecolumn


\section{Proof of Theorem \ref{thm:upper_bound}}
\label{seq:proof_of_theorem}
\noindent 
\begin{assumption}
\textbf{Transformer Property.} 
The model $f$ is $K$-Lipschitz continuous with respect to the Hausdorff distance between token sets. Formally, for any two token sets $\mathbf{X}_1, \mathbf{X}_2 \subseteq \mathbb{R}^{\mathrm{d}}$:
\begin{equation}
\|f(\mathbf{X}_1) - f(\mathbf{X}_2)\| \leq K \cdot d_H(\mathbf{X}_1, \mathbf{X}_2),
\end{equation}
where $\mathbf{d}_\mathrm{H}(\mathbf{X}_1, \mathbf{X}_2) \triangleq \max \{ \sup_{\mathbf{x}_1 \in \mathbf{X}_1} \inf_{\mathbf{x}_2 \in \mathbf{X}_2} \|\mathbf{x}_1 - \mathbf{x}_2\|, \sup_{\mathbf{x}_2 \in \mathbf{X}_2} \inf_{x_1 \in \mathbf{X}_1} \|\mathbf{x}_1 - \mathbf{x}_2\| \}$.
\label{ass:transformer_property}
\end{assumption}

\begin{theorem}
(Approximation Error Bound). Under Assumptions~\ref{ass:transformer_property}, the output difference between original and pruned token sets is bounded by:
\begin{equation}
\|f(\mathbf{X}_{\mathrm{v}}) - f(\tilde{\mathbf{X}}_\mathbf{v})\| \leq K \cdot \max_{\mathbf{x} \in \mathbf{X}_{\mathrm{v}} \setminus \tilde{\mathbf{X}}_\mathbf{v}} \min_{\tilde{\mathbf{x}} \in \tilde{\mathbf{X}}_\mathbf{v}} \|\mathbf{x} - \tilde{\mathbf{x}} \|.
\end{equation}
\end{theorem}

\noindent \textbf{Proof.}
Let $\mathbf{X}_{\mathrm{v}}$ be the initial set of visual tokens and $\tilde{\mathbf{X}}_{\mathrm{v}}$ be the retained set produced by SPpruner. Since SPpruner performs token reduction by selection, we have the subset property $\tilde{\mathbf{X}}_{\mathrm{v}} \subseteq \mathbf{X}_{\mathrm{v}}$.

First, we simplify the Hausdorff distance term $\mathbf{d}_\mathrm{H}(\mathbf{X}_{\mathrm{v}}, \tilde{\mathbf{X}}_{\mathrm{v}})$. The definition contains two symmetric terms:
\begin{itemize}
    \item The forward distance: $\sup_{\tilde{\mathbf{x}} \in \tilde{\mathbf{X}}_{\mathrm{v}}} \inf_{\mathbf{x} \in \mathbf{X}_{\mathrm{v}}} \|\tilde{\mathbf{x}} - \mathbf{x}\|$. Since every retained token $\tilde{\mathbf{x}}$ exists in the initial set $\mathbf{X}_{\mathrm{v}}$, the distance to the nearest neighbor is always 0. Thus, this term vanishes.
    \item The backward distance: $\sup_{\mathbf{x} \in \mathbf{X}_{\mathrm{v}}} \inf_{\tilde{\mathbf{x}} \in \tilde{\mathbf{X}}_{\mathrm{v}}} \|\mathbf{x} - \tilde{\mathbf{x}}\|$. This term represents the \textit{coverage radius}: the maximum distance from any discarded token to its nearest retained token.
\end{itemize}
Therefore, the error bound simplifies to:
\begin{equation}
\|f(\mathbf{X}_{\mathrm{v}}) - f(\widehat{\mathbf{X}}_{\mathrm{v}})\| \leq K \cdot \sup_{\mathbf{x} \in \mathbf{X}_{\mathrm{v}}} \min_{\tilde{\mathbf{x}} \in \tilde{\mathbf{X}}_{\mathrm{v}}} \|\mathbf{x} - \tilde{\mathbf{x}}\|.
\label{eq:simplified_bound}
\end{equation}

Next, we analyze how SPpruner's components minimize the term $\min_{\tilde{\mathbf{x}} \in \tilde{\mathbf{X}}_{\mathrm{v}}} \|\mathbf{x} - \tilde{\mathbf{x}}\|$:

\begin{enumerate}
    \item \textbf{High-Saliency Coverage via FIM:}
    The Focus Identification Module (FIM) selects a set $\mathcal{F}$ maximizing the composite score $\mathcal{S}(\mathbf{x}_\mathrm{i})$.
    For any token $\mathbf{x}$ with high intrinsic magnitude or semantic relevance, the probability of $\mathbf{x} \in \mathcal{F}$ is maximized. If $\mathbf{x} \in \mathcal{F}$, then $\|\mathbf{x} - \tilde{\mathbf{x}}\| = 0$. This ensures that critical regions contribute zero to the Hausdorff distance.

    \item \textbf{Contextual Coverage via CASSM \& SRS:}
    For the remaining tokens, the Structure-Responsive Sampling (SRS) determines the retention stride $\Delta$. The stride is dynamically modulated by the structural divergence $\delta$:
    \begin{equation}
        \Delta \propto \frac{1}{\delta} \approx \frac{1}{\text{Structural Discrepancy}}.
    \end{equation}
    Consider a discarded token $\mathbf{x}_{\text{disc}}$ located between two retained tokens $\tilde{\mathbf{x}}_\mathrm{i}$ and $\tilde{\mathbf{x}}_{\mathrm{i+1}}$. The geometric distance $\|\mathbf{x}_{\text{disc}} - \tilde{\mathbf{x}}_\mathrm{i}\|$ is bounded by the sampling interval.
    
    When the local feature space changes rapidly (high structural discrepancy), SPpruner calculates a smaller $\Delta$, forcing the sampling to be careful. This explicitly constrains the distance $\|\mathbf{x} - \tilde{\mathbf{x}}\|$ in complex regions. Conversely, in uniform background regions (low discrepancy), $\Delta$ increases, but since the feature variation $\|\mathbf{x} - \tilde{\mathbf{x}}\|$ is naturally low in uniform regions, the distance remains bounded.
\end{enumerate}

\noindent \textbf{Conclusion.}
By effectively adapting the sampling stride to the local structural context of the token distribution, SPpruner guarantees that for every $\mathbf{x} \in \mathbf{X}_{\mathrm{v}}$, there exists a proxy $\tilde{\mathbf{x}} \in \tilde{\mathbf{X}}_{\mathrm{v}}$ such that $\|\mathbf{x} - \tilde{\mathbf{x}}\| \leq \epsilon$, where $\epsilon$ is a threshold controlled by the target token budget. Substituting this into Eq. (\ref{eq:simplified_bound}), we obtain an upper bound, thereby proving that SPpruner can preserve the theoretical representational capability of the vanilla model. \hfill $\square$

\section{Detailed Experiment Settings}
\label{seq::detailed_experiment_settings}
In this section, we provided detailed information about the benchmarks and models used in our evaluation. Moreover, we have conducted more experiments to further validate SPpruner's performance.
\subsection{Evaluation Benchmarks}

To rigorously evaluate the efficacy and versatility of our proposed method, we conducted experiments on a comprehensive suite of benchmarks designed to assess the vision-language model's capabilities.

\noindent \textbf{Image Understanding Benchmarks.} We employed 13 widely recognized datasets to cover a broad spectrum of visual capabilities: MME~\citep{mme}, MMBench(EN, CN)~\citep{mmbench}, POPE~\citep{pope}, VizWiz~\citep{vizwiz}, GQA~\citep{gqa}, VQAv2~\citep{vqav2}, $\text{ScienceQA}^{\text{IMG}}$~\citep{scienceqa}, MMVet~\citep{mmvet}, MMStar~\citep{mmstar}, MMMU~\citep{mmmu}, MuirBench~\citep{muirbench} and RealWorldQA~\citep{realworldqa}.

\noindent \textbf{Chart and Document Understanding Benchmarks.} We utilized 3 challenging document understanding benchmarks to evaluate SPpruner's performance: TextVQA~\citep{textvqa}, DocVQA~\citep{docvqa}, OCRBench~\citep{ocrbench}. Also, we conduct ablation experiments on ChartQA~\citep{chartqa}.

\noindent \textbf{Image Captioning Benchmarks.} To further prove that SPpruner preserves the holistic image understanding, we conduct more experiments on Filckr30K~\citep{flickr30k}.

\noindent \textbf{Video Understanding Benchmarks.} To validate temporal reasoning capabilities, we evaluated our method on three standard video understanding benchmarks: VideoMME~\citep{videomme}, SeedBench~\citep{seedbench} and MLVU~\citep{mlvu}.

\subsubsection{Image Understanding}

\noindent \textbf{MME.} The MME benchmark is designed to rigorously evaluate a model’s perceptual and cognitive abilities through 14 subtasks. It employs carefully constructed instruction-answer pairs and concise instructions to minimize data leakage and ensure fair evaluation. This setup provides a robust measure of a model’s performance across various tasks.

\noindent \textbf{MMBench (EN, CN).} MMBench employs a hierarchical taxonomy to rigorously evaluate model performance across three granularity levels. Starting from broad perception and reasoning capabilities (L-1), the evaluation branches into six distinct sub-abilities (L-2) and ultimately expands into 20 fine-grained leaf dimensions (L-3). This coarse-to-fine structure facilitates a holistic diagnosis of multimodal proficiency. Additionally, MMBench-CN is its chinese version.

\noindent \textbf{POPE.} POPE serves as a specialized benchmark for hallucination evaluation. By posing binary inquiries regarding the presence of specific objects, it rigorously measures the model's visual faithfulness using accuracy, precision, recall, and F1 scores to detect false positive predictions.

\noindent \textbf{VizWiz.} VizWiz evaluates visual assistance capabilities using real-world images and questions collected from visually impaired users. The dataset comprises 20,523 training, 4,319 validation, and 8,000 test pairs, each with 10 human annotations. It challenges models to either provide accurate answers or detect unanswerable queries, emphasizing practical accessibility.

\noindent \textbf{GQA.} GQA integrates scene graphs, questions, and images, enriched with dense spatial and object-level attributes, to rigorously benchmark fine-grained visual reasoning and scene comprehension.

\noindent \textbf{VQAv2.} VQAv2 benchmarks open-ended visual reasoning across 265,016 real-world images, utilizing ten human-annotated answers per question to ensure robust evaluation accuracy.

\noindent \textbf{$\text{ScienceQA}^{\text{IMG}}$.} ScienceQA benchmarks multimodal reasoning across natural, language, and social sciences, organizing questions into a hierarchy of 26 topics, 127 categories, and 379 skills to test complex understanding. 

\noindent \textbf{MMVet.} MMvet assesses the integration of multimodal skills through 218 challenging samples across six core dimensions, including OCR and spatial awareness. It employs an LLM-based evaluator to standardize scoring for diverse response styles.

\noindent \textbf{MMStar.} MMStar is a curated, vision-essential benchmark consisting of 1,500 human-verified samples spanning six core capabilities and 18 fine-grained dimensions. To ensure rigorous evaluation, it strictly filters out samples solvable by text alone, thereby enforcing visual dependency to assess genuine multimodal reasoning.

\noindent \textbf{MMMU.} MMMU evaluates expert-level multimodal reasoning through 11.5K college-standard questions sourced from university exams and textbooks. Spanning six core disciplines with diverse visual inputs, it rigorously tests rigorous domain knowledge and complex reasoning.

\noindent \textbf{MuirBench.} MuirBench evaluates multi-image reasoning across 12 tasks and 10 distinct relationship types. It ensures rigorous assessment through 2,600 paired questions over 11,264 images, incorporating unanswerable perturbations to verify reliability.

\noindent \textbf{RealWorldQA.} RealWorldQA benchmarks physical and spatial reasoning capabilities, evaluating how effectively models ground their visual perception within realistic, everyday environments.

\subsubsection{Chart and Document Understanding}

\noindent \textbf{TextVQA.} TextVQA evaluates the model's capacity to comprehend scene text embedded in Open Images v3 samples~\citep{openimages}, such as signs and packaging. We utilize the validation split to test the necessary integration of OCR capabilities with visual reasoning.

\noindent \textbf{DocVQA.} DocVQA benchmarks document understanding using 50,000 questions across 12,767 diverse images, categorizing queries by reasoning type to enable granular performance analysis.

\noindent \textbf{OCRBench.} OCRBench provides a holistic assessment of OCR capabilities by aggregating 29 datasets, covering diverse tasks from basic text recognition and scene-centric VQA to key information extraction and handwritten mathematical analysis.

\noindent \textbf{ChartQA.} ChartQA evaluates complex reasoning and arithmetic skills over charts, utilizing 9.6K human-written and 23.1K machine-generated questions to test multi-step visual-logical inference. We conduct our evaluation using the standard test split.

\subsubsection{Image Captioning}

\noindent \textbf{Flickr30K.} Flickr30K comprises 31,783 images meticulously selected from the Flickr platform, with each image annotated by five independent English captions. It is designed to comprehensively evaluate models' image understanding capabilities as well as their ability to generate coherent and relevant image-to-text descriptions.

\subsubsection{Video Understanding}

\noindent \textbf{VideoMME.} VideoMME serves as a comprehensive benchmark for video MLLMs, comprising 900 manually annotated videos across 30 diverse subfields with durations ranging from 11 seconds to 1 hour. It utilizes 2,700 expert QA pairs to rigorously test sequential reasoning; notably, we exclude subtitles during our evaluation to strictly assess visual-temporal understanding.

\noindent \textbf{SeedBench.} SEEDBench evaluates generative understanding across 12 image and video dimensions through 19K human-annotated multiple-choice questions, enabling objective assessment without relying on subjective human or LLM scoring.

\noindent \textbf{MLVU.} MLVU serves as our primary testbed for long-form video comprehension, spanning durations up to 2 hours. It evaluates both holistic and fine-grained detail understanding across nine multiple-choice and generation tasks, reported via the M-Avg metric.

\subsection{Models}

We evaluate SPpruner using various open-source MLLMs. We validate our method on LLaVA family, LLaVA-1.5-7B~\citep{llava} and LLaVA-Next-7B~\citep{llava-next}, with the latter used to validate performance on high-resolution images. Moreover, to enhance the effectiveness of our proposed method, we have conducted extra experiments on advanced models with different sizes, including Qwen2.5-VL-7B and Qwen2.5-VL-32B~\citep{qwen2.5vl}. For video understanding tasks, we use LLaVA-OneVision~\citep{llavaonevision} as the baseline model. 

\subsection{Baselines}

\noindent \textbf{FastV}~\citep{fastv} targets redundancy in the shallow layers by utilizing attention maps to guide early-stage token pruning, thereby significantly mitigating computational overhead.

\noindent \textbf{PyramidDrop}~\citep{pdrop} employs a hierarchical pruning schedule across transformer layers, generating a pyramidal token distribution that optimizes the trade-off between computational efficiency and task performance.

\noindent \textbf{VisionZip}~\cite{visionzip} exploits the sparsity of visual attention by prioritizing dominant tokens and aggregating the remaining background into clustered contextual representatives, thereby maximizing visual information retention.

\noindent \textbf{Divprune}~\cite{divprune} prioritizes diversity and reformulates pruning as a Max-Min Diversity Problem (MMDP), selecting the subset that maximizes the minimum pairwise distance between retained tokens.

\noindent \textbf{DART}~\citep{dart} prioritizes redundancy reduction over importance ranking. It initializes a small set of pivot tokens and iteratively selects remaining tokens with the lowest similarity to the current set, ensuring a final output of maximally diverse visual features.

\noindent \textbf{PACT}~\citep{pact} first prunes redundant tokens based on textual cues, then employs a novel clustering strategy to aggregate tokens (within the preliminary retained set) that are closer to their centroids than a given threshold.

\subsection{Extra Experiment Results}
To further validate the efficacy of SPpruner, we conduct extensive experiments on image understanding and document understanding tasks using the Qwen2.5-VL-7B model. Table~\ref{tab:qwen2.5vl-7b} provides a comprehensive comparison of SPpruner against state-of-the-art token reduction methods across ten diverse benchmarks. Under a token budget of 33.3\%, SPpruner achieves the highest relative performance (93.4\%), effectively outperforming strong baselines like VisionZip (93.0\%) and FastV (93.1\%). Notably, SPpruner secures the best results on the challenging MMBench(EN) and MMBench(CN) datasets, demonstrating its ability to preserve critical visual information. Even under the aggressive 22.2\% pruning ratio, our method maintains remarkable robustness, achieving a top score of 80.0 on MMBench(EN) and consistently surpassing the DART baseline with an overall relative performance of 88.3\%. 

\begin{table*}[t]
\centering
\caption{Performance comparison on Qwen2.5-VL-7B under various token budgets.}
\label{tab:qwen2.5vl-7b}
\renewcommand{\arraystretch}{1.1}
\setlength{\tabcolsep}{3.5pt}
\resizebox{0.99\textwidth}{!}{
\begin{tabular}{l|cccccccccc|c}
\toprule
Methods
& MMB(EN) & MMB(CN) & MME & MMStar & POPE & MMMUVal
& MuirBench & RealWorld & TextVQA & OCRBench & Rel. \\
\midrule

\rowcolor{lightgray}
\multicolumn{12}{c}{\textit{Upper Bound: All Tokens (100\% Tokens)}} \\
\cmidrule(lr){1-12}

\textcolor{gray}{Vanilla}
& \textcolor{gray}{83.9} & \textcolor{gray}{80.8} & \textcolor{gray}{2329}
& \textcolor{gray}{62.9} & \textcolor{gray}{86.2} & \textcolor{gray}{49.4}
& \textcolor{gray}{43.5} & \textcolor{gray}{69.3} & \textcolor{gray}{77.6}
& \textcolor{gray}{845} & 100\% \\
\midrule

\rowcolor{lightgray}
\multicolumn{12}{c}{\textit{Retain 33.3\% Tokens} ($\downarrow$ 66.7\%)} \\
\midrule

FastV {\scriptsize (ECCV'24)}
& 80.1 & 76.5 & 2256 & 53.7 & 83.2 & 42.0 & 43.4 & 65.6 & 76.5 & \textbf{711} & 93.1\% \\
VisionZip {\scriptsize (CVPR'25)}
& 80.6 & 77.6 & \textbf{2270} & \textbf{57.2} & 83.2 & 41.0 & 41.5 & \textbf{66.8} & 73.8 & 700 & 93.0\% \\
DivPrune {\scriptsize (CVPR'25)}
& 80.7 & 77.4 & 2268 & 56.0 & 83.9 & 40.7 & 41.6 & 62.2 & 72.8 & 720 & 91.2\% \\
PyramidDrop {\scriptsize (CVPR'25)}
& 80.5 & 72.3 & 2246 & 56.7 & 83.6 & 41.1 & 41.8 & 65.8 & \textbf{74.2} & 612 & 90.2\% \\
PACT {\scriptsize (CVPR'25)}
& 80.6 & 77.0 & 2234 & 53.4 & 80.8 & 39.4 & 41.4 & 63.7 & 68.9 & 646 & 89.8\% \\
DART {\scriptsize (EMNLP'25)}
& 79.9 & 76.9 & 2259 & 54.7 & 83.9 & \textbf{42.1} & \textbf{42.1} & 65.7 & 72.2 & 668 & 91.1\% \\
\rowcolor{lightblue}
\textbf{SPpruner (Ours)}
& \textbf{81.1}  & \textbf{78.4}  & 2241 & 54.9 & \textbf{84.0} & 40.6
& 41.7 & 66.1 & 71.6 & 627 & \textbf{93.4\%} \\
\midrule

\rowcolor{lightgray}
\multicolumn{12}{c}{\textit{Retain 22.2\% Tokens} ($\downarrow$ 77.8\%)} \\
\midrule

FastV {\scriptsize (ECCV'24)}
& 78.8 & 75.8 & 2155 & 50.5 & 80.0 & \textbf{42.3} & 42.3 & 64.8 & \textbf{74.9} & 616 & 89.9\% \\
VisionZip {\scriptsize (CVPR'25)}
& 79.9 & 76.8 & 2219 & 55.1 & 83.3 & 41.7 & 41.3 & \textbf{66.5} & 71.2 & 620 & \textbf{91.0\%} \\
DivPrune {\scriptsize (CVPR'25)}
& 80.1 & 77.1 & 2204 & 54.2 & \textbf{83.9} & 40.6 & 41.8 & 62.6 & 71.8 & \textbf{658} & 89.8\% \\
PyramidDrop {\scriptsize (CVPR'25)}
& \textbf{80.2} & \textbf{77.8} & \textbf{2221} & \textbf{56.3} & 82.3 & 41.9 & 40.9 & 65.7 & 73.2 & 516 & 90.1\% \\
PACT {\scriptsize (CVPR'25)}
& 78.2 & 76.7 & 2173 & 48.4 & 73.3 & 37.0 & 40.1 & 54.4 & 62.7 & 535 & 83.4\% \\
DART {\scriptsize (EMNLP'25)}
& 78.2 & 75.5 & 2174 & 49.7 & 81.1 & 40.1 & 42.0 & 63.9 & 68.5 & 563 & 87.8\% \\
\rowcolor{lightblue}
\textbf{SPpruner (Ours)}
& 80.0 & 76.6 & 2182 & 50.7 & 81.6 & 39.4
& \textbf{42.9} & 62.2  & 68.7  & 570 & 88.3\textbf{\%} \\
\bottomrule
\end{tabular}}\vspace{-3ex}
\end{table*}

\section{Computational Complexity.}
Following the earlier works~\citep{fastv, vtw, divprune}, we report the computational requirements of key components, including the self-attention mechanism and the feed-forward network (FFN), for SPpruner and other baselines. Assuming a Transformer architecture with $\mathrm{T}$ layers and applying reduction at layer $\mathrm{R}$, we compute the computational cost by separately analyzing the components before and after pruning. Consequently, the total floating-point operations (FLOPs) required can be expressed as:
\begin{equation}
    \text{Total FLOPs}=\mathrm{R}\times(\mathrm{4nd^2 + 2n^2d+2ndm}) + (\mathrm{T-R})\times(\mathrm{4nd^2 + 2n^2d+2ndm}),
\end{equation}
Where $\mathrm{n}$, $\mathrm{d}$, and $\mathrm{m}$ denote the sequence length, hidden dimension, and FFN intermediate size, respectively. This equation quantifies the FLOPs reduction, degenerating to the baseline cost when $\mathrm{R=T}$. Our analysis reveals that delaying pruning to later layers offers diminishing efficiency gains. This validates our strategy of pruning at the second layer to maximize computational savings.

\section{Limitation and Future Work}
Although our method has achieved state-of-the-art performance across various benchmarks, we observe that its effectiveness is less pronounced on more advanced models such as Qwen2.5-VL compared to LLaVA. This may be attributed to the inherent token reduction mechanisms, such as patch merging, already present in these models. Investigating how to adapt to such sophisticated architectures constitutes a significant direction for our future work.

\section{Visualization}
In this section, we present visualizations of the analyzed paradigms. As shown in Figures~\ref{fig:llava_vertical_vis_part1}, ~\ref{fig:llava_vertical_vis_part2}, and ~\ref{fig:llava_vertical_vis_part3}, the reduction ratios increase from left to right, ranging from 35\% to 95\%. It is clearly observed that our proposed subject-centric reduction paradigm effectively captures the primary visual subjects of the image. Moreover, compared to the query-centric reduction paradigm, our method retains a denser set of subject tokens, which is crucial for preserving the model's holistic image understanding capability. Furthermore, as the reduction ratio increases, query-centric methods continue to suffer from token misalignment, as illustrated in the 95\% reduction scenario in Figure~\ref{fig:llava_vertical_vis_part3}. This observation further validates that our approach maintains superior visual subject recognition and performance stability even under extreme reduction ratios.

\begin{figure*}[ht]
\centering
\includegraphics[width=0.9\textwidth]{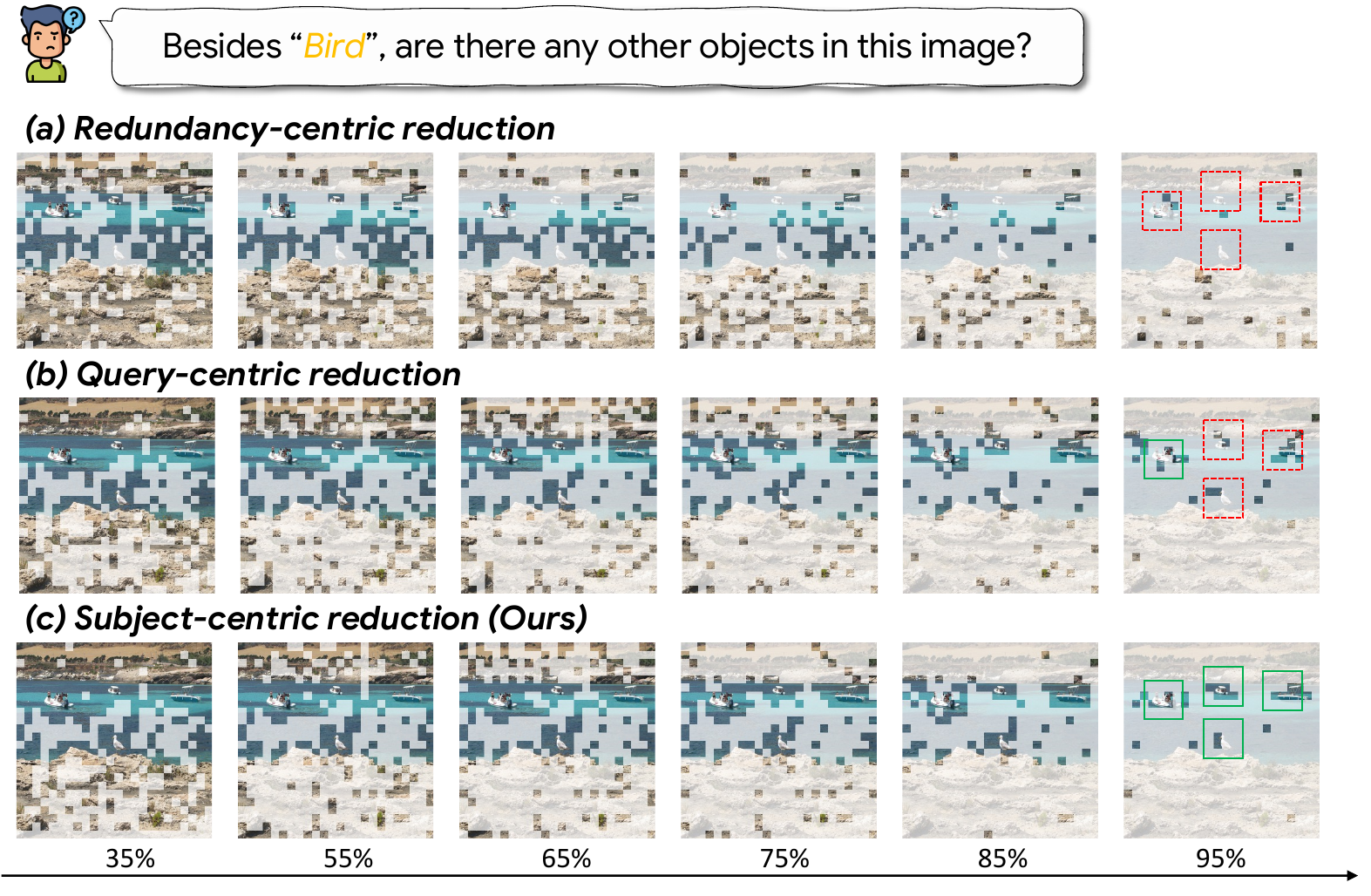}
\caption{V\textbf{Visualization of token retention across increasing reduction ratios}. Under a query inquiring about secondary objects (i.e., not the ``Bird``), SPpruner excels in capturing a broad visual subject spectrum. Unlike baselines that discard unqueried subjects, our method successfully retains salient objects (e.g., boat) even at extreme reduction ratios.}
\label{fig:llava_vertical_vis_part1}\vspace{-2ex}
\end{figure*}

\begin{figure*}[ht]

\centering
\includegraphics[width=0.9\textwidth]{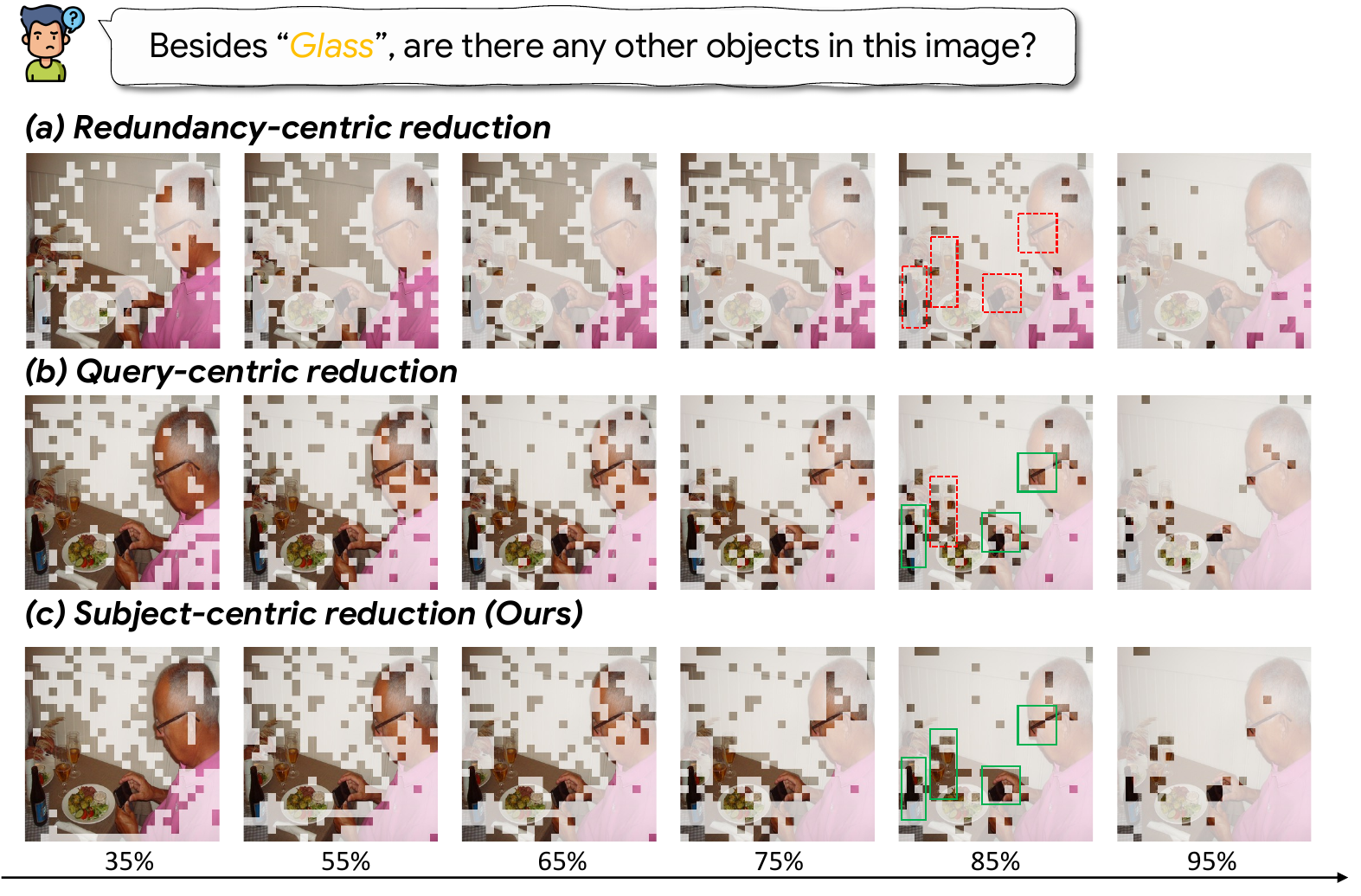}
\caption{\textbf{Visualization of token retention across increasing reduction ratios}. Under a query inquiring about secondary objects (i.e., not the ``Glass``), SPpruner excels in capturing a broad visual subjects spectrum. Unlike baselines that discard unqueried subjects, our method successfully retains salient objects (e.g., champagne) even at extreme reduction ratios.}
\label{fig:llava_vertical_vis_part2}\vspace{-2ex}
\end{figure*}

\begin{figure*}[ht]
\centering
\includegraphics[width=0.9\textwidth]{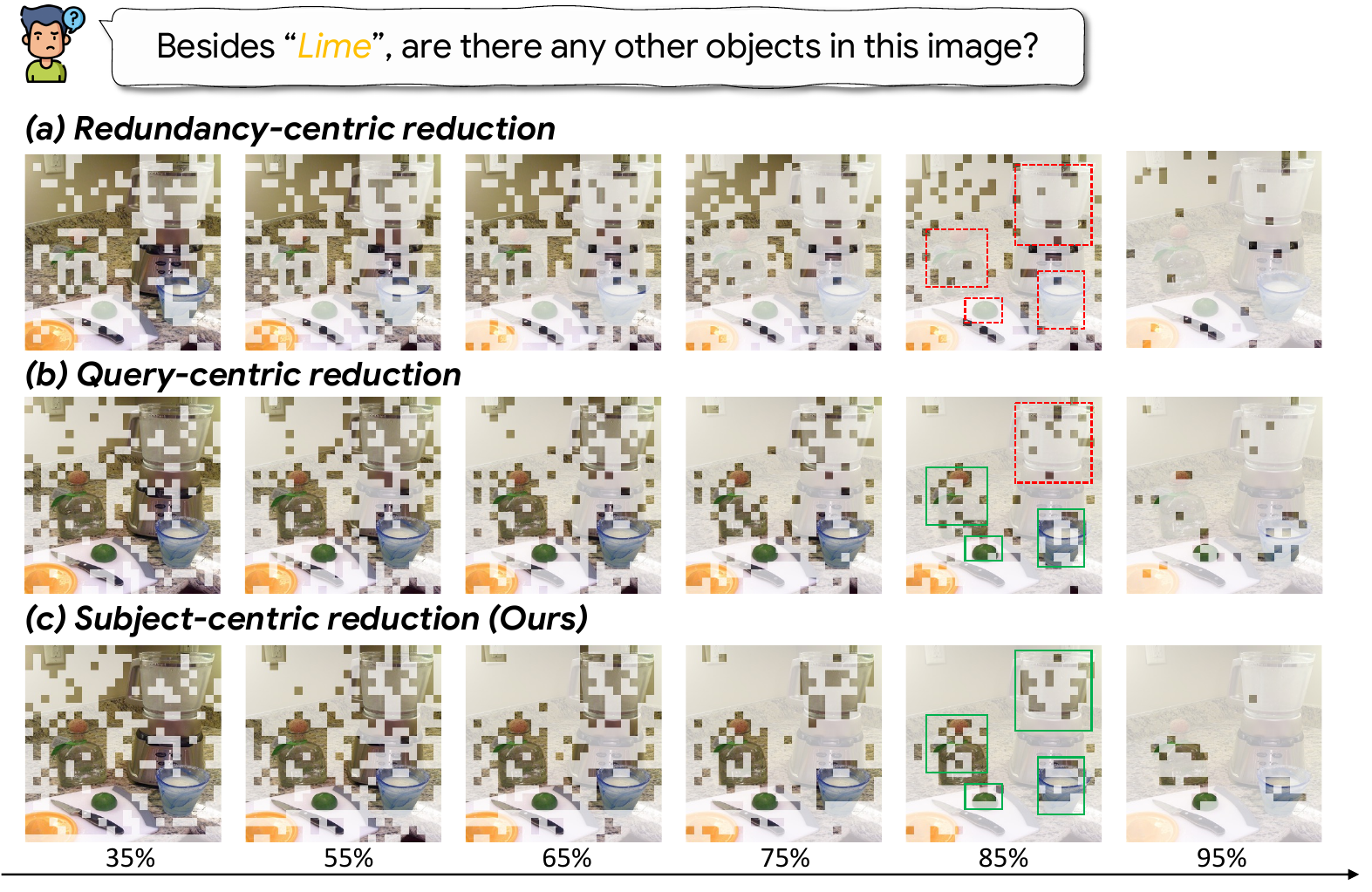}
\caption{\textbf{Visualization of token retention across increasing reduction ratios}. Under a query inquiring about secondary objects (i.e., not the ``Lime``), SPpruner excels in capturing a broad visual subjects spectrum. Unlike baselines that discard unqueried subjects, our method successfully retains salient objects (e.g., tequila) even at extreme reduction ratios.}
\label{fig:llava_vertical_vis_part3} \vspace{-2ex}
\end{figure*}



\end{document}